\documentclass{article}

\usepackage{fullpage}
\usepackage{authblk}
\usepackage[round]{natbib}
\usepackage[utf8]{inputenc} %
\usepackage[T1]{fontenc}    %
\usepackage[colorlinks=true, citecolor=green]{hyperref}       %
\usepackage{url}            %
\usepackage{booktabs}       %
\usepackage{amsfonts}       %
\usepackage{nicefrac}       %
\usepackage{microtype}      %

\usepackage{amsthm} %

\usepackage{array}
\usepackage{graphicx,subcaption}
\usepackage{algorithm}
\usepackage{mathtools}
\usepackage{wrapfig}
\usepackage{tikz-cd}
\usetikzlibrary{decorations.pathmorphing}
\usepackage{zxmacro}

\DeclareFontFamily{U}{mathx}{\hyphenchar\font45}
\DeclareFontShape{U}{mathx}{m}{n}{
      <5> <6> <7> <8> <9> <10> gen * mathx
      <10.95> mathx10 <12> <14.4> <17.28> <20.74> <24.88> mathx12
      }{}
\DeclareSymbolFont{mathx}{U}{mathx}{m}{n}
\DeclareMathSymbol{\intop}  {1}{mathx}{"B3}

\newcommand\indep{\independent}
\newcommand\independent{\protect\mathpalette{\protect\independenT}{\perp}}
\def\independenT#1#2{\mathrel{\rlap{$#1#2$}\mkern4mu{#1#2}}}

\newcommand{\wh}{\widehat}

\let\temp\phi
\let\phi\varphi
\let\varphi\temp

\newcommand{\pr}{\mathbb{P}}

\newcommand{\R}{\mathbb{R}}

\newcommand{\E}{\mathbb{E}}

\newcommand{\given}{\,|\,}  %

\usepackage{commath} %

\newcommand{\subjectto}{\text{ subject to }}

\newcommand{\DAG}{\mathsf{DAG}}
\newcommand{\score}{L}
\renewcommand{\loss}{\ell}
\newcommand{\gr}{\mathsf{G}}
\newcommand{\ver}{\mathsf{V}}
\newcommand{\edg}{\mathsf{E}}
\DeclareMathOperator{\pa}{pa}

\newcommand{\wadj}{\R^{d\times d}}

\newcommand{\dat}{\mathbf{X}}
\newcommand{\datcol}{\mathbf{x}}

\newcommand{\sob}{H^{1}}
\newcommand{\actv}{\sigma}
\newcommand{\MLP}{\mathsf{MLP}}
\newcommand{\obfcn}{\phi}
\newcommand{\obwgt}{\alpha}
\newcommand{\obmat}{\Phi}
\renewcommand{\err}{z}
\newcommand{\param}{\theta}

\renewcommand{\norm}[1]{\Vert#1\Vert}

\title{\LARGE{Learning Sparse Nonparametric DAGs}}

\author[]{Xun Zheng$^\dag$}
\author[]{Chen Dan$^\dag$}
\author[]{Bryon Aragam$^\ddag$}
\author[]{Pradeep Ravikumar$^\dag$}
\author[]{Eric P. Xing$^\dag$}
\affil[]{$^\dag$\emph{Carnegie Mellon University}, $^\ddag$\emph{University of Chicago}}

\begin{document}

\maketitle

{\let\thefootnote\relax\footnote{Contact: $^\dag$\texttt{xzheng1@andrew.cmu.edu,\{cdan,pradeepr,epxing\}@cs.cmu.edu}, $^\ddag$\texttt{bryon@chicagobooth.edu}}}

\begin{abstract}
We develop a framework for learning sparse nonparametric directed acyclic graphs (DAGs) from data. Our approach is based on a recent algebraic characterization of DAGs that led to a fully continuous program for score-based learning of DAG models parametrized by a linear structural equation model (SEM). We extend this algebraic characterization to nonparametric SEM by leveraging nonparametric sparsity based on partial derivatives, resulting in a continuous optimization problem that can be applied to a variety of nonparametric and semiparametric models including GLMs, additive noise models, and index models as special cases. 
Unlike existing approaches that require specific modeling choices, loss functions, or algorithms, we present a completely general framework that can be applied to general nonlinear models (e.g. without additive noise), general differentiable loss functions, and generic black-box optimization routines.
The code is available at \url{https://github.com/xunzheng/notears}.
\end{abstract}

\section{Introduction}
\label{sec:intro}

Learning DAGs from data is an important and classical problem in machine learning, with a diverse array of applications in causal inference \citep{spirtes2000}, fairness and accountability \citep{kusner2017}, medicine \citep{heckerman1992}, and finance \citep{sanford2012}. In addition to their undirected counterparts, DAG models offer a parsimonious, interpretable representation of a joint distribution that is useful in practice. 
Unfortunately, existing methods for learning DAGs typically rely on specific model assumptions (e.g. linear or additive) and specialized algorithms (e.g. constraint-based or greedy optimization). As a result, the burden is on the user to choose amongst many possible models and algorithms, which requires significant expertise. Thus, there is a need for a general framework for learning different DAG models---subsuming, for example, linear, parametric, and nonparametric---that does not require specialized algorithms. Ideally, the problem could be formulated as a conventional optimization problem that can be tackled with general purpose solvers, much like the current state-of-the-art for undirected graphical models \cite[e.g.][]{suggala2017expxorcist,yang2015,liu2009,hsieh2013,banerjee2008}.

In this paper, we develop such a general algorithmic framework for score-based learning of DAG models. This framework is flexible enough to learn general nonparametric dependence while also easily adapting to parametric and semiparametric models, including nonlinear models.
The framework is based on a recent algebraic characterization of acyclicity due to \citet{zheng2018dags} that recasts the score-based optimization problem as a \emph{continuous} problem, instead of the traditional combinatorial approach. This allows generic optimization routines to be used in minimizing the score, providing a clean conceptual formulation of the problem that can be approached using any of the well-known algorithms from the optimization literature. This work relies heavily on the linear parametrization in terms of a weighted adjacency matrix $W\in\Rbb^{d\times d}$, which is a stringent restriction on the class of models. One of the key technical contributions of the current work is extending this to general nonparametric problems, where no such parametrization in terms of a weighted adjacency matrix exists.

\paragraph{Contributions}
Our main contributions can be summarized as follows:
\begin{itemize}
\item We develop a generic optimization problem that can be applied to nonlinear and nonparametric SEM and discuss various special cases including additive models and index models. 
In contrast to existing work, we show how this optimization problem can be solved to stationarity with generic solvers, eliminating the necessity for specialized algorithms and models.
\item We extend the existing smooth characterization of acyclicity from \citet{zheng2018dags} to general nonparametric models, show that the linear parametrization is a special case of the general framework, and apply this to several popular examples for modeling nonlinear dependencies (Section~\ref{sec:acyclic}).
\item We consider in detail two classes of nonparametric estimators defined through 1) Neural networks and 2) Orthogonal basis expansions, and study their properties (Section~\ref{sec:opt}).
\item We run extensive empirical evaluations on a variety of nonparametric and semiparametric models against recent state-of-the-art methods in order to demonstrate the effectiveness and generality of our framework (Section~\ref{sec:exp}).
\end{itemize}

As with all score-based approaches to learning DAGs, ours relies on a nonconvex optimization problem. Despite this, we show that off-the-shelf solvers return stationary points that outperform other state-of-the-art methods.
Finally, the algorithm itself can be implemented in standard machine learning libraries such as PyTorch,
which should help 
the community to extend our approach to richer models moving forward.

\paragraph{Related work}

The problem of learning nonlinear and nonparametric DAGs from data has generated significant interest in recent years, including additive models \citep{buhlmann2014cam,voorman2014graph,ernest2016}, generalized linear models \citep{park2018learning,park2017,park2019poisson,gu2018}, additive noise models \citep{hoyer2009,peters2014causal,blobaum2018cause,mooij2016distinguishing}, post-nonlinear models \citep{zhang2009identifiability,zhang2016estimation} and general nonlinear SEM \citep{monti2019causal,goudet2018learning,kalainathan2018SAM,sgouritsa2015inference}. Recently, \citet{yu2019dag} proposed to use graph neural networks for nonlinear measurement models and \citet{huang2018generalized} proposed a generalized score function for general SEM.  The latter work is based on recent work in kernel-based measures of dependence \citep{gretton2005measuring,fukumizu2008kernel,zhang2012kernel}. 
Another line of work uses quantile scoring \citep{tagasovska2018nonparametric}.
Also of relevance is the literature on nonparametric variable selection \citep{bertin2008selection,lafferty2008rodeo,miller2010local,rosasco2013nonparametric,gregorova2018structured} and approaches based on neural networks \citep{feng2017sparse,ye2018variable,abid2019concrete}. The main distinction between our work and previous work is that our framework is not tied to a specific model---as in \cite{yu2019dag,buhlmann2014cam,park2018learning}---as our focus is on a \emph{generic} formulation of an optimization problem that can be solved with \emph{generic} solvers (see Section~\ref{sec:background} for a more detailed comparison). 
This also distinguishes this paper from concurrent work by  \citet{lachapelle2019gradient} that focuses on neural network-based nonlinearities in the local conditional probabilities. 
Furthermore, compared to \citet{huang2018generalized} and \citet{yu2019dag}, our approach can be much more efficient (Section~\ref{sec:exp:structure}; Appendix~\ref{sec:supp:additional-figures}).
As such, we hope that this work is able to spur future work using more sophisticated nonparametric estimators and optimization schemes.

\paragraph{Notation}
Norms will always be explicitly subscripted to avoid confusion: $\norm{\cdot}_{p}$ is the $\ell_{p}$-norm on vectors, $\norm{\cdot}_{L^{p}}$ is the $L^{p}$-norm on functions, $\norm{\cdot}_{p,q}$ is the $(p,q)$-norm on matrices, and $\norm{\cdot}_{F}=\norm{\cdot}_{2,2}$ is the matrix Frobenius norm. For functions $f:\R^{s}\to\R$ and a matrix $A\in\R^{n\times s}$, we adopt the convention that $f(A)\in\R^{n}$ is the vector whose $i$th element is $f(a^{i})$, where $a^{i}$ is the $i$th row of $A$.

\section{Background}
\label{sec:background}

Our approach is based on (acyclic) structural equation models as follows. Let $X=(X_{1},\ldots,X_{d})$ be a random vector and $\gr=(\ver,\edg)$ a DAG with $\ver=X$. We assume that there exist functions $f_{j}:\R^{d}\to\R$
\footnote{The reason for writing $f_{j}(X)$ instead of $f_{j}(X_{\pa(j)})$ is to simplify notation by ensuring each $f_{j}$ is defined on the same space.} and $g_{j}:\R\to\R$ such that
\begin{align}
\E[X_{j}\given X_{\pa(j)}]
= g_{j}(f_{j}(X)), \quad \E f_{j}(X) = 0
\label{eq:sparse:np}
\end{align}
and $ f_{j}(u_{1},\ldots,u_{d}) $  does not depend on $ u_{k} $  if   $ X_{k}\notin\pa(j) $, where $\pa(j)$ denotes the parents of $X_{j}$ in $\gr$. 
Formally, the independence statement means that for any $X_{k}\notin\pa(j)$, the function $a(u):=f_{j}(X_{1},\ldots,X_{k-1},u,X_{k+1},X_{d})$ is constant for all $u\in\R$. Thus, $\gr$ encodes the conditional independence structure of $X$. The functions $g_{j}$, which are typically known, allow for possible non-additive errors such as in generalized linear models (GLMs). The model \eqref{eq:sparse:np} is quite general and includes additive noise models, linear and generalized linear models, and additive models as special cases (Section~\ref{sec:acyclic:cases}).

In this setting, the DAG learning problem can be stated as follows: Given a data matrix $\dat=[\datcol_{1}\given\cdots\given\datcol_{d}]\in\R^{n\times d}$ consisting of $n$ i.i.d. observations of the model \eqref{eq:sparse:np}, we seek to learn the DAG $\gr(X)$ that encodes the dependency between the variables in $X$. Our approach is to learn $f=(f_{1},\ldots,f_{d})$ such that $\gr(f)=\gr(X)$ using a score-based approach. Given a loss function $\loss(y,\yhat)$ such as least squares or the negative log-likelihood, we consider the following program: 
\begin{align}
\begin{aligned}
\min_{f} \ \score(f)
\ \subjectto \
\gr(f)\in\DAG, \\
\text{where} \quad 
\score(f)
= \frac1n\sum_{j=1}^{d}\loss(\datcol_{j}, f_{j}(\dat)). 
\end{aligned}
\label{eq:scorebased:nonlinear}
\end{align}
There are two challenges in this formulation: 1) How to enforce the acyclicity constraint that $\gr(f)\in\DAG$, and 2) How to enforce sparsity in the learned DAG $\gr(f)$? Previous work using linear and generalized linear models rely on a parametric representation of $\gr$ via a weighted adjacency matrix $W\in\wadj$, which is no longer well-defined in the model \eqref{eq:sparse:np}. To address this, we develop a suitable surrogate of $W$ defined for general nonparametric models, to which we can apply the trace exponential regularizer from \citet{zheng2018dags}.

\subsection{Identifiability} 
Existing papers approach this problem as follows: 1) Assume a specific model for \eqref{eq:sparse:np}, 2) Prove identifiability for this specific model, and 3) Develop a specialized algorithm for learning this specific model. By contrast, our approach is generic: We do not assume any particular model form or algorithm, and instead develop a general framework that applies to any model that is identifiable. By now, there is a well-catalogued list of identifiability results for various linear, parametric, and nonlinear models, which we review briefly below (see also Section~\ref{sec:acyclic:cases}). 

When the model \eqref{eq:sparse:np} holds, the graph $\gr$ is not necessarily uniquely defined: A well-known example is when $X$ is jointly normally distributed, in which case the $f_{j}$ are linear functions, and where it can be shown that the graph $\gr$ is not uniquely specified. Fortunately, it is known that this case is somewhat exceptional: Assuming additive noise, as long as the $f_{j}$ are linear with non-Gaussian errors \citep{kagan1973,shimizu2006,loh2014causal} or the functions $f_{j}$ are nonlinear \citep{hoyer2009,zhang2009,peters2014causal}, then the graph $\gr$ is generally identifiable. We refer the reader to \citet{peters2014causal} for details. Another example are so-called \emph{quadratic variance function} models, which are parametric models that subsume many generalized linear models \citep{park2017,park2018learning}. In the sequel, we assume that the model is chosen such that the graph $\gr$ is uniquely defined from \eqref{eq:sparse:np}, and this dependence will be emphasized by writing $\gr=\gr(X)$. Similarly, any collection of functions $f=(f_{1},\ldots,f_{d})$ defines a graph $\gr(f)$ in the obvious way. See Section~\ref{sec:acyclic:cases} for specific examples with
discussion on identifiability.

\subsection{Comparison to existing approaches}
It is instructive at this point to highlight the main distinction between our approach and existing approaches. A common approach is to assume the $f_{j}$ are easily parametrized (e.g. linearity) \citep{zheng2018dags,aragam2015,gu2018,park2017,park2018learning,chen2018causal,ghoshal2017ident}. In this case, one can easily encode the structure of $\gr$ via, e.g. a weighted adjacency matrix, and learning $\gr$ reduces to a parametric estimation problem. Nonparametric extensions of this approach include additive models \citep{buhlmann2014cam,voorman2014graph}, where the graph structure is easily deduced from the additive structure of the $f_{j}$. More recent work \citep{lachapelle2019gradient,yu2019dag} uses specific parametrizations via neural networks to encode $\gr$. An alternative approach relies on exploiting the conditional independence structure of $X$, such as the post-nonlinear model \citep{zhang2009identifiability,yu2019dag}, the additive noise model \citep{peters2014causal}, and kernel-based measures of conditional independence \citep{huang2018generalized}. Our framework can be viewed as a substantial generalization of these approaches: We use partial derivatives to measure dependence in the \emph{general} nonparametric model \eqref{eq:sparse:np} without assuming a particular form or parametrization, and do not explicitly require any of the machinery of nonparametric conditional independence (although we note in some places this machinery is implicit). This allows us to use nonparametric estimators such as multilayer perceptrons and basis expansions, for which these derivatives are easily computed. As a result, the score-based learning problem is reduced to an optimization problem that can be tackled using existing techniques, making our approach easily accessible.

\section{Characterizing acyclicity in nonparametric SEM}
\label{sec:acyclic}

In this section, we discuss how to extend the trace exponential regularizer from \cite{zheng2018dags} beyond the linear setting, and then discuss several special cases.

\subsection{Linear SEM and the trace exponential regularizer}
\label{sec:acyclic:linear}

We begin by briefly reviewing \cite{zheng2018dags} in the linear case, i.e. $g_{j}(s)=s$ and $f_{j}(X)=w_{j}^{T}X$ for some $w_{j}\in\Rbb^{d}$. This defines a matrix $W=[w_{1}\given\cdots\given w_{d}]\in\wadj$ that precisely encodes the graph $\gr(f)$, i.e. there is an edge $X_{k}\to X_{j}$ in $\gr(f)$ if and only if $w_{kj}\ne 0$. In this case, we can formulate the entire problem in terms of $W$: If $\score(W)=\norm{\dat-\dat W}_{F}^{2}/(2n)$, then optimizing $\score(W)$ is equivalent to optimizing $\score(f)$ over linear functions.
Define the function $h(W)=\tr e^{W\circ W}-d$, where $[W\circ W]_{kj}=w_{kj}^{2}$. Then \cite{zheng2018dags} show that \eqref{eq:scorebased:nonlinear} is equivalent to 
\begin{align}
\label{eq:scorebased:notears:lin}
\min_{W\in\wadj} \score(W) 
\ \subjectto \ 
h(W) = 0,
\end{align}
The key insight from \cite{zheng2018dags} is replacing the combinatorial constraint $\gr(W)\in\DAG$ with the continuous constraint $h(W)=0$.
Our goal is to define a suitable surrogate of $W$ for general nonparametric models, so that the same continuous program can be used to optimize \eqref{eq:scorebased:nonlinear}.

\subsection{A notion of nonparametric acyclicity}
\label{sec:acyclic:np}

Unfortunately, for general models of the form \eqref{eq:sparse:np}, there is no $W$, and hence the trace exponential formulation seems to break down. To remedy this, we use partial derivatives to measure the dependence of $f_{j}$ on the $k$th variable, an idea that
dates back to at least \citet{rosasco2013nonparametric}. First, we need to make precise the spaces we are working on: Let $\sob(\R^{d})\subset L^{2}(\R^{d})$ denote the usual Sobolev space of square-integrable functions whose derivatives are also square integrable (for background on Sobolev spaces see \cite{tsybakov2009introduction}). Assume hereafter that $f_{j}\in\sob(\R^{d})$ and denote the partial derivative with respect to $X_{k}$ by $\partial_{k}f_{j}$. It is then easy to show that $f_{j}$ is independent of $X_{k}$ if and only if $\norm{\partial_{k}f_{j}}_{L^{2}}=0$, where $\norm{\cdot}_{L^{2}}$ is the usual $L^{2}$-norm. This observation implies that the matrix $ W(f)
 = W(f_{1},\ldots,f_{d})
\in\wadj $ with entries
\begin{align}
\begin{aligned}
& [W(f)]_{kj}
:= \norm{\partial_{k}f_{j}}_{L^{2}}.
\end{aligned}
\label{eq:defn:npW}
\end{align}
precisely encodes the dependency structure amongst the $X_{j}$.
Thus the program \eqref{eq:scorebased:nonlinear} is equivalent to
\begin{align}
\label{eq:scorebased:notears:np}
\min_{f: f_{j}\in\sob(\R^{d}), \forall j \in [d]}\score(f)
\ \subjectto \
h(W(f)) = 0.
\end{align}
This implies an equivalent continuous formulation of the program \eqref{eq:scorebased:nonlinear}. Moreover, when the functions $f_{j}$ are all linear, $W(f)$ is the same as the weighted adjacency matrix $W$ defined in Section~\ref{sec:acyclic:linear}. Thus, \eqref{eq:scorebased:notears:np} is a genuine generalization of the linear case \eqref{eq:scorebased:notears:lin}.

\subsection{Special cases}
\label{sec:acyclic:cases}

In addition to applying to general nonparametric models of the form \eqref{eq:scorebased:nonlinear} and linear models, the program \eqref{eq:scorebased:notears:np} applies to a variety of parametric and semiparametric models including additive noise models, generalized linear models, additive models, and index models. In this section we discuss these examples along with identifiability results for each case.

\paragraph{Additive noise models}
The nonparametric additive noise model (ANM) \citep{hoyer2009,peters2014causal} assumes that 
\begin{align}
\label{eq:anm}
&X_{j} = f_{j}(X) + \err_{j}, \ 
\E f_{j}(X) = 0, \
\err_{j}\indep f_{j}(X).
\end{align}
and $ \err_{j}\sim\pr_{j} $ is the random noise. 
Clearly this is a special case of \eqref{eq:sparse:np} with $g_{j}(s)=s$. In contrast to the remaining examples below, without additional assumptions, it is not possible to simplify the condition for $[W(f)]_{kj}=0$ in \eqref{eq:defn:npW}. Assuming the $f_{j}$ are three times differentiable and not linear in any of its arguments, this model is identifiable \citep[Corollary~31]{peters2014causal}.

\paragraph{Generalized linear models}
A traditional GLM assumes that $\E[X_{j}\given X_{\pa(j)}] = g_{j}(w_{j}^{T}X)$ for some known link functions $g_{j}:\R\to\R$ and $w_{j}\in\R^{d}$. For example, we can use logistic regression for $X_{j}\in\{0,1\}$ with  $g_{j}(s)=e^{s}/(1+e^{s})$. This is easily generalized to nonparametric mean functions $f_{j}\in\sob(\R^{d})$ by setting
\begin{align}
\begin{aligned}
\label{eq:glm}
&\E[X_{j}\given X_{\pa(j)}]
= g_{j}(f_{j}(X)).
\end{aligned}
\end{align}
Clearly, \eqref{eq:anm} is a special case of \eqref{eq:glm}. Furthermore, for linear mean functions, $[W(f)]_{kj}=0$ if and only if $w_{jk}=0$, recovering the parametric approach in \citet{zheng2018dags}. Several special cases of GLMs are known to be identifiable: Linear Gaussian with equal variances \citep{peters2013}, linear non-Gaussian models \citep{shimizu2006}, Poisson models \citep{park2019poisson}, and quadratic variance function models \citep{park2017}.

\paragraph{Polynomial regression}
In polynomial regression, we assume that $f_{j}(X)$ is a polynomial in $X_{1},\ldots,X_{d}$. More generally, given a known dictionary of functions $\eta_{\ell}(u_{1},\ldots,u_{d})$, we require that $f_{j}(X)=\sum_{\ell}\beta_{j\ell}\,\eta_{\ell}(X)$. Then it is easy to check that $[W(f)]_{kj}=0$ if and only if $\beta_{j\ell}=0$ whenever $\eta_{\ell}$ depends on $u_{k}$. For each $k$, define $a_{jk}(u):=f_{j}(X_{1},\ldots,X_{k-1},u,X_{k+1},X_{d})$. As long as $a_{jk}(u)$ is not a linear function (i.e. each $f_{j}$ is a degree-2 polynomial or higher in $X_{k}$) for all $k$ and $j$, then Corollary~31 in \citet{peters2014causal} implies identifiability of this model.

\paragraph{Additive models}
In an additive model \citep{hastie1987generalized,ravikumar2009sparse}, we assume that $f_{j}(X)=\sum_{k\ne j}f_{jk}(X_{k})$ for some $f_{jk}\in\sob(\R)$. Then it is straightforward to show that $\norm{\partial_{k}f_{j}}_{L^{2}}=0$ if and only if $f_{jk}=0$. In other words, $[W(f)]_{kj}=0$ if and only if $\norm{f_{jk}}_{L^{2}}=0$. Assuming the $f_{jk}$ are three times differentiable and not linear in any of its arguments, this model is identifiable \citep[Corollary~31, see also \citealp{buhlmann2014cam}]{peters2014causal}.

\paragraph{Index models}
The multiple index model \citep{alquier2013,yuan2011} assumes $f_{j}(X)=\sum_{m=1}^{M}h_{jm}(\beta_{jm}^{T}X)$ for some $h_{jm}\in\sob(\R)$ and $\beta_{jm}\in\R^{d}$. As long as $M$ is sufficiently large, these functions are universal approximators \citep{diaconis1984nonlinear}. When $M=1$, this is known as a single-index model. As long as the functions $h_{jm}$ ($m=1,\ldots,M$) are linearly independent, it is straightforward to show that $\norm{\partial_{k}f_{j}}_{L^{2}}=0$ if and only if $\beta_{jmk}=0$ for each $m$. In other words, $[W(f)]_{kj}=0$ if and only if $\sum_{m=1}^{M}\beta_{jmk}^{2}=0$. Once again, assuming three-times differentiability and nonlinearity of $h_{jm}$, Corollary~31 in \citet{peters2014causal} implies identifiability of this model.

Among these examples, both polynomial regression and GLMs with linear mean function are nonlinear but finite-dimensional, and hence the problem \eqref{eq:scorebased:notears:np} is straightforward to solve (see Section~\ref{sec:opt:solve}).

\section{Optimization}
\label{sec:opt}

In general, the program \eqref{eq:scorebased:notears:np} is infinite-dimensional. In this section we discuss different ways to reduce this to a tractable, finite-dimensional optimization problem. One of the advantages of encoding dependence via $W(f)$ is that it provides a plug-and-play framework for plugging in various nonparametric estimators whose derivatives can be computed. We will illustrate two examples using multilayer perceptrons and orthogonal basis expansions, however, we emphasize that it is straightforward to implement other differentiable models for the $f_{j}$. These flexible nonparametric estimators will help reduce \eqref{eq:scorebased:notears:np} to a straightforward optimization problem, as we discuss at the end of this section.

The basic recipe is the following:
\begin{enumerate}
\item Choose a model family for the conditional expectations $\E[X_{j}\given X_{\pa(j)}]$ (e.g. general nonparametric, additive, index, etc.);
\item Choose a suitable family of approximations (e.g. neural networks, orthogonal series, etc.);
\item Translate the loss function $\score(f)$ and constraint $W(f)$ into parametric forms $\score(\param)$ and $W(\param)$ using the approximating family;
\item Solve the resulting finite-dimensional problem.
\end{enumerate}
Step 3 above is the key step that enables transforming \eqref{eq:scorebased:notears:np} into a tractable optimization problem. By approximating the $f_{j}$ with a flexible family of functions parametrized by $\param$, we can replace the infinite-dimensional quantity $W(f)$ with the simpler $W(\param)$. As is standard in the literature on nonparametric estimation, the dimension of $\param$ is allowed to depend on $n$, although this dependence will be suppressed.

\subsection{Multilayer perceptrons}
\label{sec:opt:mlp}

We first consider the use of neural networks to approximate the $f_{j}$, as in an ANM \eqref{eq:anm} or GLM \eqref{eq:glm}. Consider a multilayer perceptron (MLP) with $h$ hidden layers and a single activation $\actv:\R\to\R$, given by 
\begin{align*}
& \MLP(u;A^{(1)},\ldots,A^{(h)})
= \actv(A^{(h)}\actv(\cdots A^{(2)}\actv(A^{(1)}u))), \\
& \quad A^{(\ell)}\in\R^{m_{\ell}\times m_{\ell-1}},
\quad
m_{0}=d.
\end{align*}
By increasing the capacity of the MLP (e.g. increasing the number of layers $h$ or the number of hidden units $m_{\ell}$ in each layer), we can approximate any $f_{j}\in\sob(\R^{d})$ arbitrarily well.

First, we must determine under what conditions $\MLP(u;A^{(1)},\ldots,A^{(h)})$ is independent of $u_{k}$---this is important both for enforcing acyclicity and sparsity. It is not hard to see that if the $k$th column of $A^{(1)}$ consists of all zeros (i.e. $A^{(1)}_{b k}=0$ for all $b=1,\ldots,m_{1}$), then $\MLP(u;A^{(1)},\ldots,A^{(h)})$ will be independent of $u_{k}$.
In fact, we have the following proposition, which implies that this constraint precisely identifies the set of MLPs that are independent of $u_{k}$:
\begin{proposition}
\label{lem:mlp:indep}
Consider the function class $\mathcal{F} $ of all MLPs that are independent of $u_{k}$ and the function class $\mathcal{F}_0 $ of all MLPs such that the $k$th column of $A^{(1)}$ consists of all zeros. Then $\mathcal{F} = \mathcal{F}_{0}$.
\end{proposition}

This important proposition 
provides a rigorous way to enforce that an MLP approximation depends only on a few coordinates. 
Indeed, it is clear that constraining $A^{(1)}_{b k}=0$ for each $b$ will remove the dependence on $k$, however, there is a concern that we could lose the expressivity of multiple hidden layers in doing so. Fortunately, this proposition implies that there is in fact no loss of expressivity or approximating power. Furthermore, it follows that $[W(f)]_{kj}=0$ if 
$ \norm{ k\text{th-column}(A_{j}^{(1)}) }_2 = 0 $. 
This result enables us to characterize acyclicity independent of the depth of the neural network, as opposed to handling individual paths through the entire neural network as in \cite{lachapelle2019gradient}, which depends linearly on the depth.

Let $\param_{j}=(A_{j}^{(1)},\ldots,A_{j}^{(h)})$ denote the parameters for the $j$th MLP and $\param=(\param_{1},\ldots,\param_{d})$. 
Define $[W(\param)]_{kj}= \norm{ k\text{th-column}(A_{j}^{(1)} ) }_2 $. 
The problem \eqref{eq:scorebased:nonlinear} thus reduces to
\begin{align}
\min_{\theta} \quad  &  \ \frac1{n}\sum_{j=1}^{d}
    \loss(\datcol_{j}, \MLP(\dat; \theta_j)  )  + \lambda\norm{A_{j}^{(1)}}_{1,1} \nonumber \\
\subjectto &  \  h(W(\param)) = 0.
\label{eq:scorebased:notears:mlp}
\end{align}

\subsection{Basis expansions}
\label{sec:opt:sobolev}

As an alternative to neural networks, we also consider the use of orthogonal basis expansions \citep{schwartz1967estimation,wahba1981data,hall1987cross,efromovich2008nonparametric}. While many techniques are valid, we adopt an approach based on \citet{ravikumar2009sparse}. Let $\{\obfcn_{r}\}_{r=1}^{\infty}$ be an orthonormal basis of $\sob(\R^{d})$ such that $\E\obfcn_{r}(X)=0$ for each $r$. Then any $f\in\sob(\R^{d})$ can be written uniquely
\begin{align}
f(u)
&= \sum_{r=1}^{\infty}\obwgt_{r}\obfcn_{r}(u),
\quad
\obwgt_{r}
= \int_{\R^{d}}\obfcn_{r}(u)f(u)\,du.
\end{align}
As long as the coefficients $\obwgt_{r}$ decay sufficiently fast, $f$ can be well-approximated by the finite series $\wh{f}^{R}:=\sum_{r=1}^{R}\obwgt_{r}\obfcn_{r}$. Similar claims are true for one-dimensional Sobolev functions, which applies to both additive (i.e. for $f_{jk}$) and index (i.e. for $h_{jm}$) models.

We illustrate here an application with additive models and one-dimensional expansions. It is straightforward to extend these ideas to more general models using a tensor product basis, though this quickly becomes computationally infeasible. For more on high-dimensional orthogonal series, see \citet{lee2016spectral}. Thus, 
\begin{align}
\begin{aligned}
f_{j}(u_{1},\ldots,u_{d})
& =\sum_{k\ne j}f_{jk}(u_{k}) \\
& =\sum_{k\ne j}\sum_{r=1}^{\infty}\obwgt_{jkr}\obfcn_{r}(u_{k}).
\end{aligned}
\end{align}
Given integers $R_{k}$ and assuming $f_{jk}$ is sufficiently smooth, we have $\norm{f_{jk}-\wh{f}_{jk}^{R_{k}}}_{L^{2}}=O(1/R_{k})$ \citep{efromovich2008nonparametric}, so that the overall approximation error is on the order $O(d/\min_{k}R_{k})$. Furthermore, $[W(f)]_{kj}=0\iff\norm{f_{jk}}_{L^{2}}=0\iff\obwgt_{jkr}=0$ for all $r$. Since we are discarding terms for $r>R_{k}$, in practice it suffices to check that $\obwgt_{jkr}=0$ for $r=1\,\ldots,R_{k}$, or $\sum_{r=1}^{R_{k}}\obwgt_{jkr}^{2}=0$.

Letting $\param$ denote the parameters $\obwgt_{jkr}$ for all $j,k,r$, it thus suffices to define $[W(\param)]_{kj}=[\sum_{r=1}^{R_{k}}\obwgt_{jkr}^{2}]^{1/2}$ for the purposes of checking acyclicity.
Let $\obmat_{k}$ be the matrix $[\obmat_{k}]_{ir}=\obfcn_{r}(X_{k}^{(i)})$. To estimate the coefficients $\obwgt_{jkr}$, we solve
\begin{align}
\label{eq:basis:spam}
\min_{\theta} \quad  & \
\frac1{n}\sum_{j=1}^{d}
    \loss\Big(\datcol_{j}, \sum_{k\ne j}\obmat_{k}a_{jk}\Big) \nonumber \\
&     + \lambda_{1}\sum_{k\ne j}\frac1n a_{jk}^{T}\obmat_{k}^{T}\obmat_{k}a_{jk} 
    + \lambda_{2}\sum_{k\ne j}\norm{a_{jk}}_{1}  \nonumber \\
\subjectto & \  h(W(\param)) = 0.
\end{align}
This is similar to \citet{ravikumar2009sparse} with the addition of an explicit $\ell_{1}$ constraint. 

\subsection{Solving the continuous program}
\label{sec:opt:solve}

Having converted $ L(f) $ and $ W(f) $ to their finite-dimensional counterparts, we are now ready to solve \eqref{eq:scorebased:notears:np} by applying standard optimization techniques. 
We emphasize that the hard work went into formulating the programs \eqref{eq:scorebased:notears:mlp} and \eqref{eq:basis:spam} as generic problems for which off-the-shelf solvers can be used.
Importantly, since in both \eqref{eq:scorebased:notears:mlp} and \eqref{eq:basis:spam} the term $ W(\theta) $ is differentiable \wrt $ \theta $, the optimization program is an $ \ell_1 $-penalized smooth minimization under a differentiable equality constraint.
As in \citet{zheng2018dags}, the standard machinery of augmented Lagrangian can be applied, resulting in a series of unconstrained problems: 
\begin{align}
\label{eq:opt:augmented-primal}
\begin{aligned}
& \min_{\theta} \ F(\theta) + \lambda \norm{\theta}_1, & \\
& F(\theta) = L(\theta) + \frac{\rho}{2} |h (W(\theta))|^2 + \alpha h(W(\theta)) &
\end{aligned}
\end{align}
where $ \rho$ is a penalty parameter and $ \alpha $ is a dual variable. 

A number of optimization algorithms can be applied to the above \emph{unconstrained} $ \ell_1 $-penalized smooth minimization problem. 
A natural choice is the L-BFGS-B algorithm~\citep{byrd1995limited}, which can be directly applied by casting \eqref{eq:opt:augmented-primal} into a box-constrained form:
\begin{align}
\begin{aligned}
& \min_{\theta} \  F(\theta) + \lambda \norm{\theta}_1  \\
& \iff \ \ 
\min_{\theta^+\ge 0, \theta^- \ge 0}   F(\theta^+ - \theta^-) + \lambda \one^T  (\theta^+ + \theta^-)
\end{aligned}
\label{eq:opt:box-constrained-form}
\end{align}
where $ \one $ is a vector of all ones. We note that as in \citet{zheng2018dags}, \eqref{eq:opt:augmented-primal} is a nonconvex program, and at best can be solved to stationarity. Our experiments indicate that this nonetheless leads to competitive and often superior performance in practice.

\section{Experiments}
\label{sec:exp}

\newcommand{\mlp}{\mathsf{NOTEARS\mbox{-}MLP}}
\newcommand{\mlppp}{\mathsf{NOTEARS\mbox{-}MLP}\texttt{++}}
\newcommand{\basis}{\mathsf{NOTEARS\mbox{-}Sob}}
\newcommand{\basispp}{\mathsf{NOTEARS\mbox{-}Sob}\texttt{++}}
\newcommand{\linear}{\mathsf{Linear}}
\newcommand{\fgs}{\mathsf{FGS}}
\newcommand{\gs}{\mathsf{GSGES}}
\newcommand{\gnn}{\mathsf{GNN}}
\newcommand{\cam}{\mathsf{CAM}}

\begin{figure*}[t]
\centering
\includegraphics[width=0.99\textwidth]{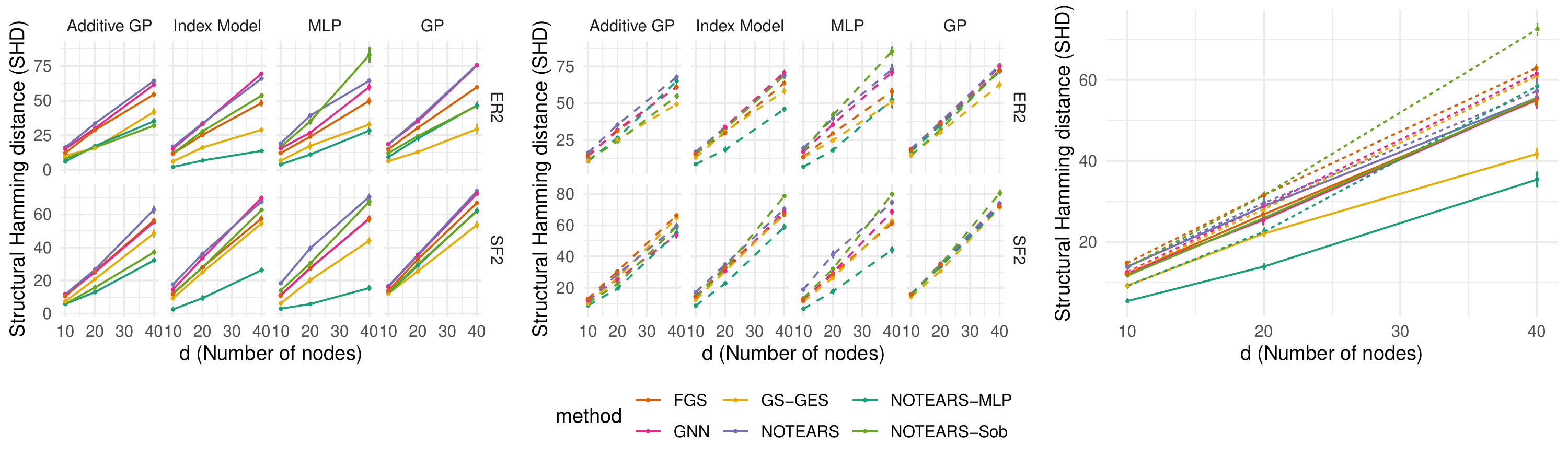}
\caption{Structure recovery measured by SHD (lower is better) to ground truth. Left: $ n = 1000 $. Middle: $ n = 200 $. Right: Average over all configurations. Rows: random graph model (Erdos-Renyi and scale-free). Columns: different types of SEM. $ \mlp $ performs well on a wide range of settings, while $ \basis $ shows good accuracy on additive models.
}
\label{fig:compare:shd}
\end{figure*}

\begin{figure*}[t]
\centering
\includegraphics[width=0.99\textwidth]{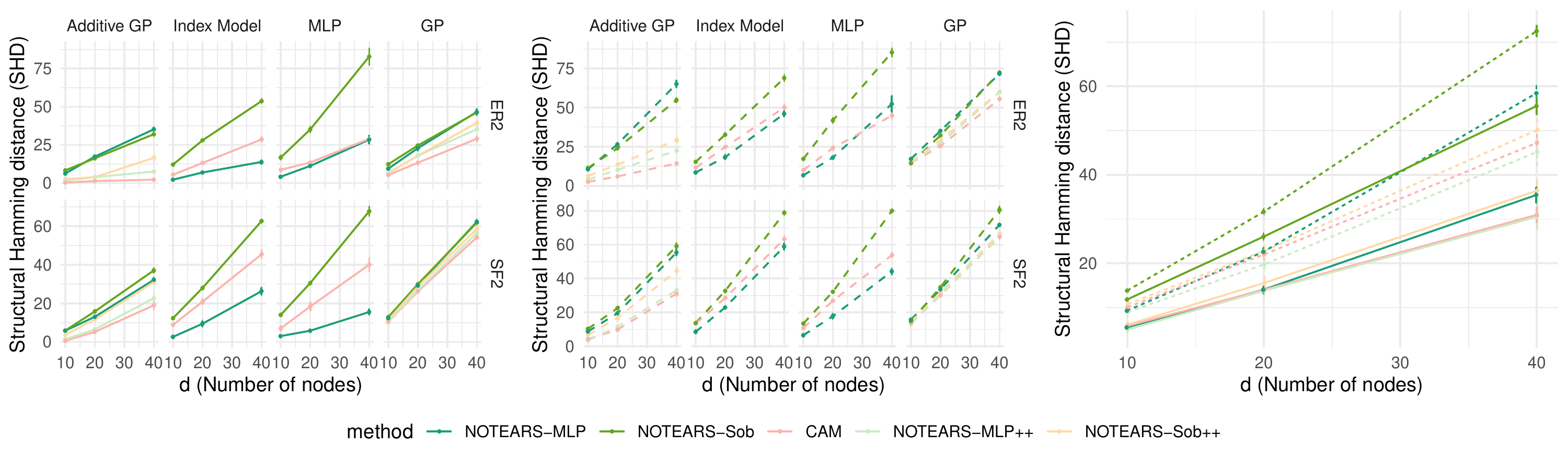}
\caption{Structure recovery measured by SHD (lower is better) to ground truth. Left: $ n = 1000 $. Middle: $ n = 200 $. Right: Average over all configurations. Rows: random graph model (Erdos-Renyi and scale-free). Columns: different types of SEM. Either $ \mlp $ or $ \mlppp $ (i.e. $ \mlp $ with neighborhood selection and pruning) achieves competitive accuracy compared to $ \cam $. 
}
\label{fig:compare:shd:cam}
\end{figure*}

We study the empirical performance of two instances of the general framework: MLP (\ref{sec:opt:mlp}) and Sobolev expansions (\ref{sec:opt:sobolev}), denoted by $ \mlp $ and $ \basis $. 
For $ \mlp $ we use an MLP with one hidden layer with 10 hidden units and sigmoid activation function. 
For $ \basis $ we use Sobolev basis $ \phi_r (u) = s_r \sin(u / s_r) $, $ s_r = 2 / ((2r - 1)\pi) $ ($ r = 1, \dotsc, 10 $).
Complete details on all baselines and simulations, including a discussion of computational complexity and runtimes, can be found 
in the appendix.
Code implementing our method is available at \url{https://github.com/xunzheng/notears}.

\paragraph{Baselines}
For comparison, the following methods are chosen as baselines:
fast greedy equivalence search $ (\fgs) $~\citep{ramsey2017million}, 
greedy equivalence search with generalized scores $ (\gs) $~\citep{huang2018generalized}, 
DAG-GNN $ (\gnn) $~\citep{yu2019dag}, 
NOTEARS $ (\linear) $~\citep{zheng2018dags} for linear SEM, 
and causal additive models $ (\cam) $~\citep{buhlmann2014cam}. 
To summarize, $ \fgs $ and $ \linear $ are specialized at linear models, whereas $ \gs $, $ \gnn $, and $ \cam $ targets general nonlinear dependencies.
Comparisons with other score-based methods (KGV score~\citep{bach2003learning}, Spearman correlation~\citep{sokolova2014causal}) and constraint-based methods (PC~\citep{spirtes2000}, MM-MB~\citep{aliferis2010local}) can be found in previous work \citep{huang2018generalized}, hence are omitted. 

\paragraph{Simulation}
The ground truth DAG is generated from two random graph models: Erdos-Renyi (ER) and scale-free (SF). 
We use ER2 to denote an ER graph with $ s_0 = 2d $ edges, likewise for SF. 
Given the ground truth DAG, we simulate the SEM $ X_j = f_j (X_{\pa(j)}) + z_j $ for all $ j \in [d] $ in topological order, and each $ z_j \sim \Nsf(0,1) $. 
To evaluate the performance under different data generation mechanisms, we consider four models for the $ f_j $: 1) Additive models with Gaussian processes (GPs) for each $f_{jk}$, 2) Index models ($M=3$), 3) ANM with MLPs, and 4) ANM with GPs.

\paragraph{Metrics}
We evaluate the estimated DAG structure using the following common metrics: false discovery rate (FDR), true positive rate (TPR), false positive rate (FPR), and structural Hamming distance (SHD). 
Note that both $ \fgs $ and $ \gs $ return a CPDAG that may contain undirected edges, in which case we evaluate them favorably by assuming correct orientation for undirected edges whenever possible, similar to \citep{zheng2018dags}.

\subsection{Structure learning}
\label{sec:exp:structure}

In this experiment we examine the structure recovery of different methods by comparing the DAG estimates against the ground truth. 
We simulate \{ER1, ER2, ER4, SF1, SF2, SF4\} graphs with $ d = \{10, 20, 40\} $ nodes.
For each graph, $ n = \{1000, 200\} $ data samples are generated. 
The above process is repeated 10 times and we report the mean and standard deviations of the results. 
For $ \mlp $ and $ \basis $, $ \lambda = \{0.01, 0.03\} $ are used for $ n = \{1000, 200\} $ respectively.

Figure~\ref{fig:compare:shd} shows the SHD in various settings; the complete set of results for the remaining metrics are deferred to the supplement. 
Overall, the proposed $ \mlp $ method attains the best SHD (lower the better) across a wide range of settings, particularly when the data generating mechanism is an MLP or an index model. 
One can also observe that the performance of $ \mlp $ stays stable for different graph types with varying density and degree distribution, 
as it does not make explicit assumptions on the topological properties of the graph such as density or degree distribution. 
Not surprisingly, $ \basis $ performs well when the underlying SEM is additive GP. 
On the other hand, when the ground truth is not an additive model, the performance of $ \basis $ degrades as expected. 
Finally, we observe that $ \gs $ outperforms $ \mlp $ and $ \basis $ on GP, which is a nonparametric setting in which a kernel-based dependency measure can excel,
however, we note that the kernel-based approach accompanies an $ O(n^3) $ time complexity, compared to linear dependency on $ n $ in $ \mlp $ and $ \basis $.
Also, with by properly tuning the regularization parameter, the performance of $ \mlp $ for each individual setting can be improved considerably, for example in the GP setting.
Since such hyperparameter tuning is not the main focus of this paper, we fix a reasonable $ \lambda $ for all settings (see Appendix~\ref{sec:supp:additional-figures} for more discussion). 

With respect to runtime and scalability, we note that the computational complexity of our approach depends on the choice of nonparametric estimator. For example, $ \mlp $ requires $ O(n d^2 m + d^2 m + d^3) $ flops per iteration of L-BFGS-B.
In terms of runtime, the average runtime of $ \gs $ on ER2 with $d = 40$, $n = 1000$ is over 90 minutes, whereas $ \mlp $ takes less than five minutes on average (see Appendix~\ref{sec:supp:additional-figures} for more discussion).

Figure~\ref{fig:compare:shd:cam} shows the SHD compared with $ \cam $. 
We first observe that $ \mlp $ outperforms $ \cam $ in multiple index models and MLP models,
on the other hand, $ \cam $ achieves better accuracy on additive GP and the full GP setting. 
Recall that the $\cam$ algorithm involves three steps: 1) Preliminary neighborhood search (PNS), 2) Order search by greedy optimization of the likelihood, and 3) Edge pruning. By comparison, our methods effectively only perform the second step, and can easily be pre- and post-processed with the first (PNS) and third (edge pruning) steps.
To further investigate the efficacy of these additional steps, we applied both preliminary neighborhood selection and edge pruning to $ \mlp $ and $ \basis $ on additive GP and GP settings, denoted as $ \mlppp $ and $ \basispp $.
Noticeably, the output from PNS simply translates to a set of constraints in the form of $\theta_j = 0$ that can be easily incorporated into the L-BFGS-B algorithm for \eqref{eq:opt:box-constrained-form}, demonstrating the flexibility of the proposed approach.
The performance improves in both cases, matching or improving vs. $ \cam $.

\subsection{Sensitivity to number of hidden units}

\begin{figure}
\centering 
\includegraphics[width=0.5\textwidth]{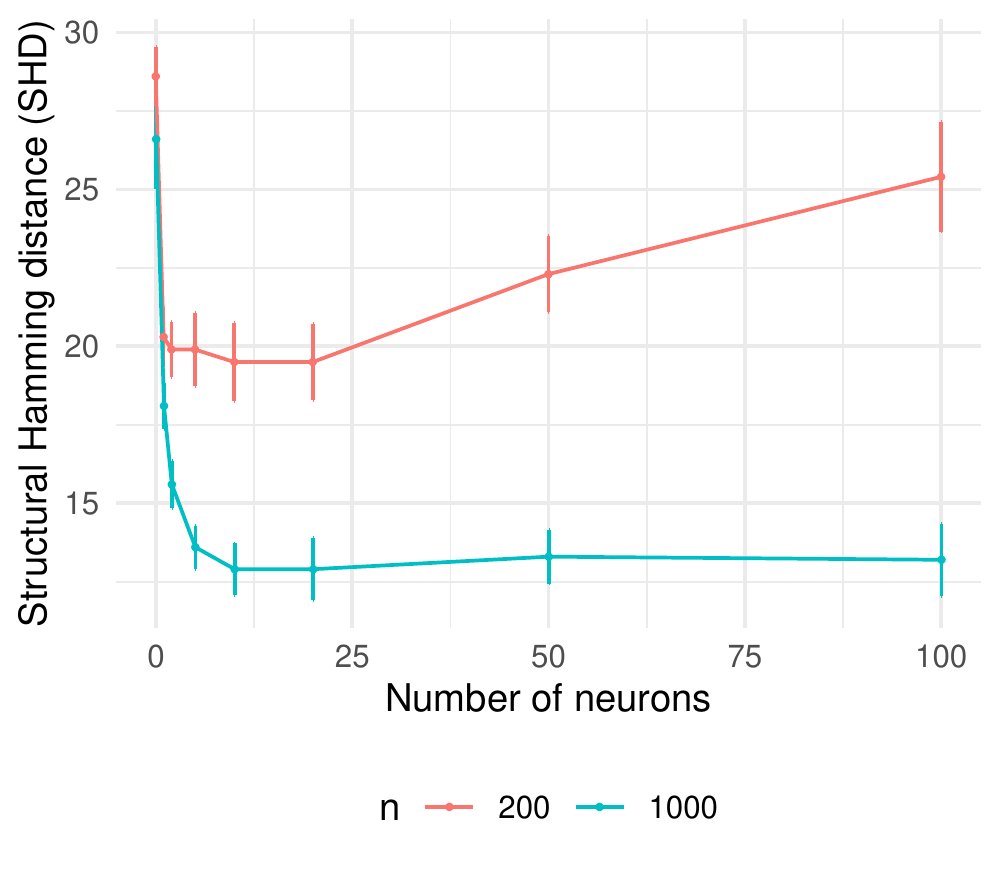}
\caption{SHD (lower is better) with varying hidden layer size in $ \mlp $. }
\label{fig:sensitivity}
\end{figure}

We also investigated the effect of number of hidden units in the $ \mlp $ estimate. 
It is well-known that as the size of the hidden layer increases, the functions representable by an MLP become more flexible. 
On the other hand, larger networks 
require more samples to estimate the parameters. 
Indeed, Figure~\ref{fig:sensitivity} confirms this intuition.
We plot the SHD with varying number of hidden units ranging from zero (\ie linear function) to 100 units, using $ n=1000 $ and $ n=200 $ samples generated from the additive GP model on SF2 graph with $ d=20 $ nodes.
One can first observe a sharp phase transition between zero and very few hidden units, which suggests the power of nonlinearity. 
Moreover, as the number of hidden units increases to 20, the performance for both $ n=1000 $ and $ n=200 $ steadily improves, in which case the increased flexibility brings benefit. 
However, as we further increase the number of hidden units, while SHD for $ n=1000 $ remains similar, the SHD for $ n=200 $ deteriorates, hinting at the lack of samples to take advantage of the increased flexibility.

\subsection{Real data}
Finally, we evaluated $ \mlp $ on a real dataset from \cite{sachs2005causal} that is commonly used as a benchmark as it comes with a \emph{consensus network} that is accepted by the biological community. 
The dataset consists of $ n=7466 $ continuous measurements of expression levels of proteins and phospholipids in human immune system cells for $ d=11 $ cell types. 
We report an SHD of 16 with 13 edges estimated by $ \mlp $. 
In comparison, 
NOTEARS predicts 16 edges with SHD of 22 and
$ \gnn $ predicts 18 edges that attains SHD of 19. (Due to the large number of samples, we could not run $\gs$ on this dataset.)
Among the 13 edges predicted by $ \mlp $, 7 edges agree with the consensus network: 
raf $ \to $ mek, 
mek $ \to $ erk,
PLCg $ \to $ PIP2, 
PIP3 $ \to $ PLCg,
PIP3 $ \to $ PIP2,
PKC $ \to $ mek,
PKC $ \to $ jnk; 
and 3 edges are predicted but in a reversed direction:
raf $ \gets $ PKC,
akt $ \gets $ erk,
p38 $ \gets $ PKC. 
Among the true positives, 3 edges are not found by other methods:
mek $ \to $ erk,
PIP3 $ \to $ PLCg,
PKC $ \to $ mek.

\section{Discussion}
\label{sec:disc}
We present a framework for score-based learning of sparse  directed acyclic graphical models that subsumes many popular parametric, semiparametric, and nonparametric models as special cases. 
The key technical device is a notion of nonparametric acyclicity that leverages partial derivatives in the algebraic characterization of DAGs. 
With a suitable choice of the approximation family, the estimation problem becomes a finite-dimensional differentiable program that can be solved by standard optimization algorithms. 
The resulting continuous optimization algorithm updates the entire graph (i.e. all edges simultaneously) in each iteration using global information about the current state of the network, as opposed to traditional local search methods that update one edge at a time based on local information. 
Notably, our approach is generally more efficient and more accurate than existing approaches, despite relying on generic algorithms. This out-of-the-box performance is desirable, especially when noting that future improvements and specializations can be expected to improve the approach substantially.

\paragraph{Acknowledgements}

We acknowledge the support of NSF via IIS-1909816, OAC-1934584, ONR via N000141812861, NSF IIS1563887 and DARPA/AFRL FA87501720152. 
Any opinions, findings and conclusions or recommendations expressed in this material are those of the author(s) and do not necessarily reflect the views of the National Science Foundation, Defense Advanced Research Projects Agency, or Air Force Research Laboratory.

\bibliographystyle{abbrvnat}

\bibliography{dagmlp-bib}

\clearpage
\appendix

\section{Proofs}

In this Appendix, we prove Proposition~\ref{lem:mlp:indep}. For completeness, note that
\begin{align*}
    \mathcal{F} = \{ f \ |\  & f(u) =  \MLP(u;A^{(1)},\ldots,A^{(h)}), \\
    & f \text{ independent of } u_k \}
\end{align*}
and
\begin{align*}
    \mathcal{F}_0 = \{ f \ |\  & f(u) =  \MLP(u;A^{(1)},\ldots,A^{(h)}), \\
    & A^{(1)}_{ bk }=0, \forall b=1,\ldots,m_{1} \}.
\end{align*}
We omit the bias terms in each layer as it does not affect the statement. 

\begin{proof}[Proof of Proposition~\ref{lem:mlp:indep}]
We will show that $\mathcal{F} \subseteq \mathcal{F}_0$ and $\mathcal{F}_0 \subseteq \mathcal{F}$. 

(1) $\mathcal{F}_0 \subseteq \mathcal{F}$: 
for any $f_0 \in \mathcal{F}_0$, we have $f_0(u) =  \MLP(u;A^{(1)},\ldots,A^{(h)}) $, where $ A^{(1)}_{bk}=0$ for all $ b = 1, \dotsc, m_1 $. 
Hence the linear function $A^{(1)}u$ is independent of $u_k$. Therefore, 
\begin{align*}
    f_0(u) & =  \MLP(u;A^{(1)},\ldots,A^{(h)}) \\
    & = \actv(A^{(h)}\actv(\cdots A^{(2)}\actv(A^{(1)}u)))
\end{align*}
is also independent of $u_k$, which means $f_0 \in \mathcal{F}$. 

(2) $\mathcal{F} \subseteq \mathcal{F}_0$: for any $f \in \mathcal{F}$, we have $f(u) =  \MLP(u;A^{(1)},\ldots,A^{(h)})$ and $f$ is independent of $u_k$. We will show that $f \in \mathcal{F}_0$ by constructing a matrix $\tilde{A}^{(1)}$, such that
\begin{align}
    f(u) =  \MLP(u;\tilde{A}^{(1)},{A}^{(2)},\ldots,A^{(h)})
\end{align}
and $ \tilde{A}_{bk}^{(1)} = 0$ for all $b = 1, \dotsc, m_1 $.

Let $\Tilde{u}$ be the vector such that $\Tilde{u}_k= 0$ and $\Tilde{u}_{k'} = u_k$ for all $k' \neq k$. 
Since $\Tilde{u}$ and $u$  differ only on the $k$th dimension, and $f$ is independent of $u_k$, we have 
\begin{align}\label{eqn:f_u_f_tilde_u}
    f(u) = f(\Tilde{u}) =  \MLP(\Tilde{u};A^{(1)},\ldots,A^{(h)}).
\end{align}
Now define $\tilde{A}^{(1)}$ be the matrix such that $\tilde{A}_{ bk }^{(1)} = 0$ and $\tilde{A}_{ bk' }^{(1)} = {A}_{bk}^{(1)}$ for all $ k'\neq k$. 
Then we have the following observation: for each entry $ s \in \{1, \dotsc, m_1\}$, 
\begin{align*}
(\tilde{A}^{(1)} u)_s 
& = \sum_{k'=1}^{d} \tilde{A}_{sk'} u_{k'}
= \sum_{k' \neq k} A_{sk'} u_{k'} \\
& = \sum_{k'=1}^{d} A_{s k'}\tilde{u}_{k'} 
= ({A}^{(1)} \tilde{u})_s.
\end{align*}
Hence,
\begin{align}
    \tilde{A}^{(1)} u = {A}^{(1)} \tilde{u}.
\end{align}
Therefore, by \eqref{eqn:f_u_f_tilde_u}
\begin{align*}
    f(u) & = f(\tilde{u}) \\
    &= \MLP(\tilde{u};A^{(1)},\ldots,A^{(h)}) \\
    &= \actv(A^{(h)}\actv(\cdots A^{(2)}\actv(A^{(1)}\tilde{u})))\\
    &= \actv(A^{(h)}\actv(\cdots A^{(2)}\actv(\tilde{A}^{(1)}u))) \\
    &= \MLP(u;\tilde{A}^{(1)},A^{(2)}, \ldots,A^{(h)})
\end{align*}
By definition of $\mathcal{F}_0$, we know that $\MLP(u;\tilde{A}^{(1)},A^{(2)}, \ldots,A^{(h)}) \in \mathcal{F}_0$. Thus, $f \in \mathcal{F}_0$ and we have completed the proof.
\end{proof}

\section{Experiment details}
\label{sec:supp:experiment-details}

\paragraph{Baselines}
We consider the following baselines.
\begin{itemize}
\item Fast greedy equivalence search $ (\fgs) $\footnote{\url{https://github.com/bd2kccd/py-causal}}~\citep{ramsey2017million} is based on greedy search and assumes linear dependency between variables. 
\item Greedy equivalence search with generalized scores $ (\gs) $\footnote{\url{https://github.com/Biwei-Huang/Generalized-Score-Functions-for-Causal-Discovery/}}~\citep{huang2018generalized} is also based on greedy search, but uses generalized scores without assuming a particular model class. 
\item DAG-GNN $ (\gnn) $\footnote{\url{https://github.com/fishmoon1234/DAG-GNN}}~\citep{yu2019dag} learns a (noisy) nonlinear transformation of a linear SEM using neural networks. 
\item NOTEARS $ (\linear) $\footnote{\url{https://github.com/xunzheng/notears}}~\citep{zheng2018dags} learns a linear SEM using continuous optimization. 
\item Causal additive model $ (\cam) $\footnote{\url{https://cran.r-project.org/package=CAM}}~\citep{buhlmann2014cam} learns an additive SEM by leveraging efficient nonparametric regression techniques and greedy search over edges. 
\end{itemize}
For all experiments, default parameter settings are used, except for $ \cam $ where both preliminary neighborhood selection and pruning are applied.

\paragraph{Simulation}
Given the graph $ \Gsf $, we simulate the SEM $ X_j = f_j (X_{\pa(j)}) + z_j $ for all $ j \in [d] $ in the topological order induced by $ \Gsf $. 
We consider the following instances of $ f_j $: 
\begin{itemize}
\item Additive GP: $ f_j(X_{\pa(j)}) = \sum_{k \in \pa(j)} f_{jk} (X_k) $, where each $ f_{jk} $ is a draw from Gaussian process with RBF kernel with length-scale one.
\item Index model: $ f_j(X_{\pa(j)}) = \sum_{m=1}^{3} h_{m} (\sum_{k \in \pa(j)} \theta_{jmk} X_k) $, where $ h_1 = \tanh$, $h_2 = \cos$, $h_3 = \sin $, and each $ \theta_{jmk} $ is drawn uniformly from range $ [-2, -0.5] \cup [0.5, 2] $. 
\item MLP: $ f_j $ is a randomly initialized MLP with one hidden layer of size 100 and sigmoid activation.
\item GP:  $ f_j $ is a draw from Gaussian process with RBF kernel with length-scale one.
\end{itemize}
In all settings, $ z_j $ is \iid standard Gaussian noise.

\section{Additional results}
\label{sec:supp:additional-figures}

\paragraph{Full comparison}
We show \{SHD, FDR, TPR, FPR\} results on all \{ER1, ER2, ER4, SF1, SF2, SF4\} graphs in Figure~\ref{fig:compare:shd:supp}, \ref{fig:compare:fdr}, \ref{fig:compare:tpr}, \ref{fig:compare:fpr} respectively. 
Similarly, see Figure~\ref{fig:compare:shd:cam:supp}, \ref{fig:compare:fdr:cam}, \ref{fig:compare:tpr:cam}, \ref{fig:compare:fpr:cam} for full comparison with $ \cam $. 
As in Figure~\ref{fig:compare:shd}, each row is a random graph model, each column is a type of SEM. 
Overall $ \mlp $ has low FDR/FPR and high TPR, and same for $ \basis $ on additive GP. 
Also observe that in most settings $ \gnn $ has low FDR as well as low TPR, which is a consequence of only predicting a small number of edges.

\paragraph{Complexity and runtime}

\begin{table*}[t]
\centering
\begin{tabular}{@{}lllllll@{}}
\toprule
           & $ \mlp $           & $ \basis $         & $ \fgs $        & $ \linear $       & $ \gnn $           & $ \gs $              \\ \midrule
$ d = 20 $ & 92.12 $\pm$ 22.51  & 62.90 $\pm$ 16.83  & 0.55 $\pm$ 0.43 & 10.95 $\pm$ 4.52  & 498.32 $\pm$ 43.72 & 1547.42 $\pm$ 109.83 \\
$ d = 40 $ & 282.64 $\pm$ 67.46 & 321.88 $\pm$ 57.33 & 0.59 $\pm$ 0.17 & 43.15 $\pm$ 12.43 & 706.35 $\pm$ 64.49 & 6379.98 $\pm$ 359.67 \\ \bottomrule
\end{tabular}
\caption{Runtime (in seconds) of various algorithms on ER2 graph with $ n = 1000 $ samples.}
\label{tbl:runtime}
\end{table*}

Recall that numerical evaluation of matrix exponential involves solving linear systems, hence the time complexity is typically $O(d^3)$ for a dense $d \times d$ matrix. 
Taking $ \mlp $ with one hidden layer of $m$ units as an example, it takes $ O(n d^2 m + d^2 m + d^3) $ time to evaluate the objective and the gradient. 
If $ m/d = O(1)$, this is comparable to the linear case $ O(n d^2 + d^3) $, except for the inevitable extra cost from using a nonlinear function. 
This highlights the benefit of Proposition~\ref{lem:mlp:indep}: the acyclicity constraint almost comes for free. Furthermore, we used a quasi-Newton method to reduce the number of calls to evaluate the gradient, which involves computing the matrix exponential. 
Table~\ref{tbl:runtime} contains runtime comparison of different algorithms on ER2 graph with $ n = 1000 $ samples. 
Recall that the kernel-based approach of $ \gs $ comes with a $ O(n^3) $ computational complexity, whereas $ \mlp $ and $ \basis $ has $ O(n) $ dependency on $ n $. 
This can be confirmed from the table, which shows $ \gs $ has a significantly longer runtime.

\paragraph{Comments on hyperparameter tuning}
The experiments presented in this paper were conducted under a fixed (and therefore suboptimal) value of $ \lambda $ and weight threshold across all graph types, sparsity levels, and SEM types, despite the fact that each configuration may prefer different regularization strengths. 
Indeed, we observe substantially improved performance by choosing different values of hyperparameters in some settings. 
As our focus is not on attaining the best possible accuracy in all settings by carefully tuning the hyperparameters, we omit these results in the main text and only include here as a supplement. 
For instance, for ER4 graph with $ d=40 $ variables and $  n = 200 $ samples, when the SEM is additive GP and MLP, setting $ \lambda = 0.03  $ and threshold = 0.5  gives results summarized in Table~\ref{tbl:rerun1}. 

\begin{table*}[t]
\centering
\begin{tabular}{@{}lllllll@{}}
\toprule
SEM         & Method & SHD                & FDR             & TPR             & FPR             & Predicted \#       \\ \midrule
Additive-GP & $\mlp$ & 124.3 $\pm$ 6.65   & 0.30 $\pm$ 0.07 & 0.35 $\pm$ 0.04 & 0.04 $\pm$ 0.01 & 81.70 $\pm$ 10.49  \\
            & $\gs$  & 121.3 $\pm$ 5.02   & 0.36 $\pm$ 0.05 & 0.28 $\pm$ 0.03 & 0.04 $\pm$ 0.00 & 69.30 $\pm$ 5.01   \\
MLP         & $\mlp$ & 88.40 $\pm$ 11.29  & 0.18 $\pm$ 0.08 & 0.57 $\pm$ 0.06 & 0.03 $\pm$ 0.02 & 111.70 $\pm$ 15.97 \\
            & $\gs$  & 121.60 $\pm$ 11.95 & 0.33 $\pm$ 0.09 & 0.33 $\pm$ 0.06 & 0.04 $\pm$ 0.01 & 77.10 $\pm$ 7.13   \\ \bottomrule
\end{tabular}
\caption{ER4, $ d=40 $, $ n=200 $ with $ \lambda = 0.03 $ and threshold = 0.5.}
\label{tbl:rerun1}
\end{table*}

\begin{figure*}[t]
\centering
\includegraphics[width=0.99\textwidth]{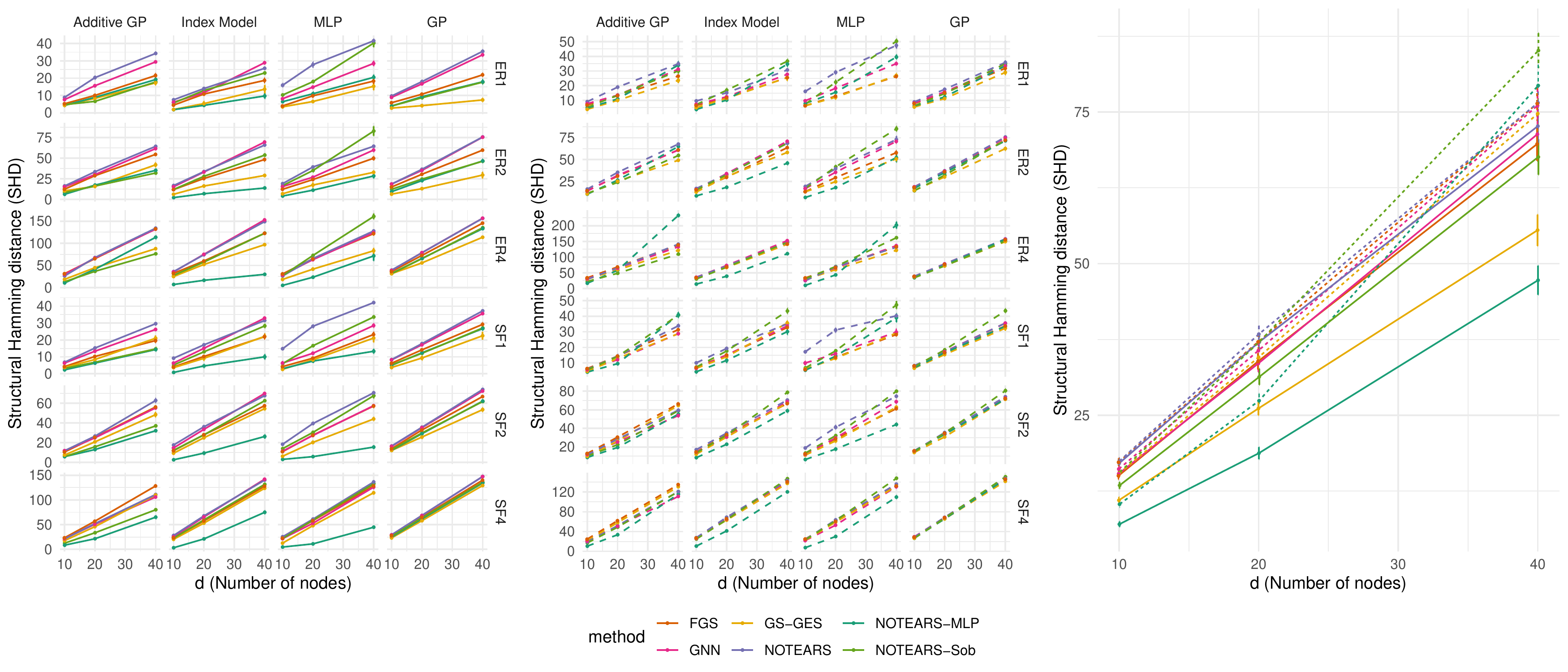}
\caption{Structure recovery measured by SHD (lower is better) to ground truth.}
\label{fig:compare:shd:supp}
\end{figure*}

\begin{figure*}[t]
\centering
\includegraphics[width=0.99\textwidth]{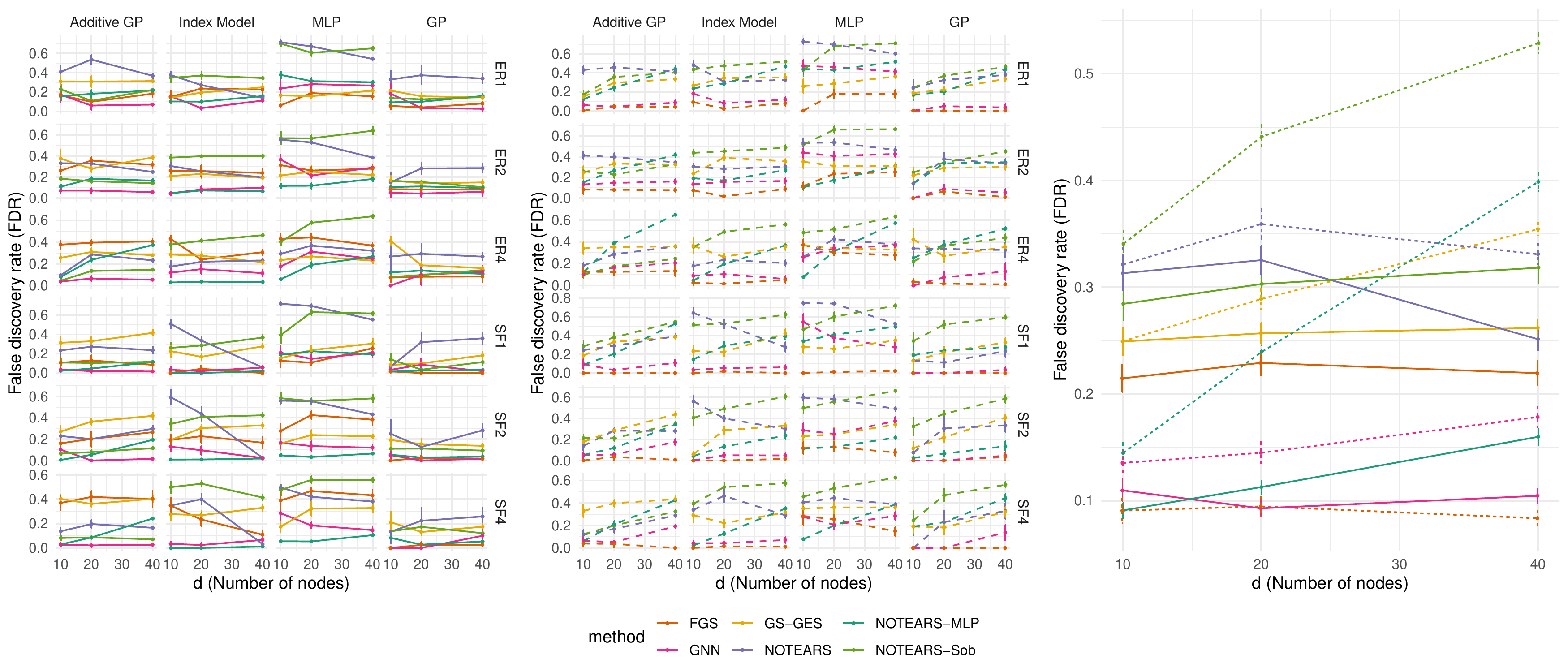}
\caption{Structure recovery measured by FDR (lower is better) to ground truth.}
\label{fig:compare:fdr}
\end{figure*}

\begin{figure*}[t]
\centering
\includegraphics[width=0.99\textwidth]{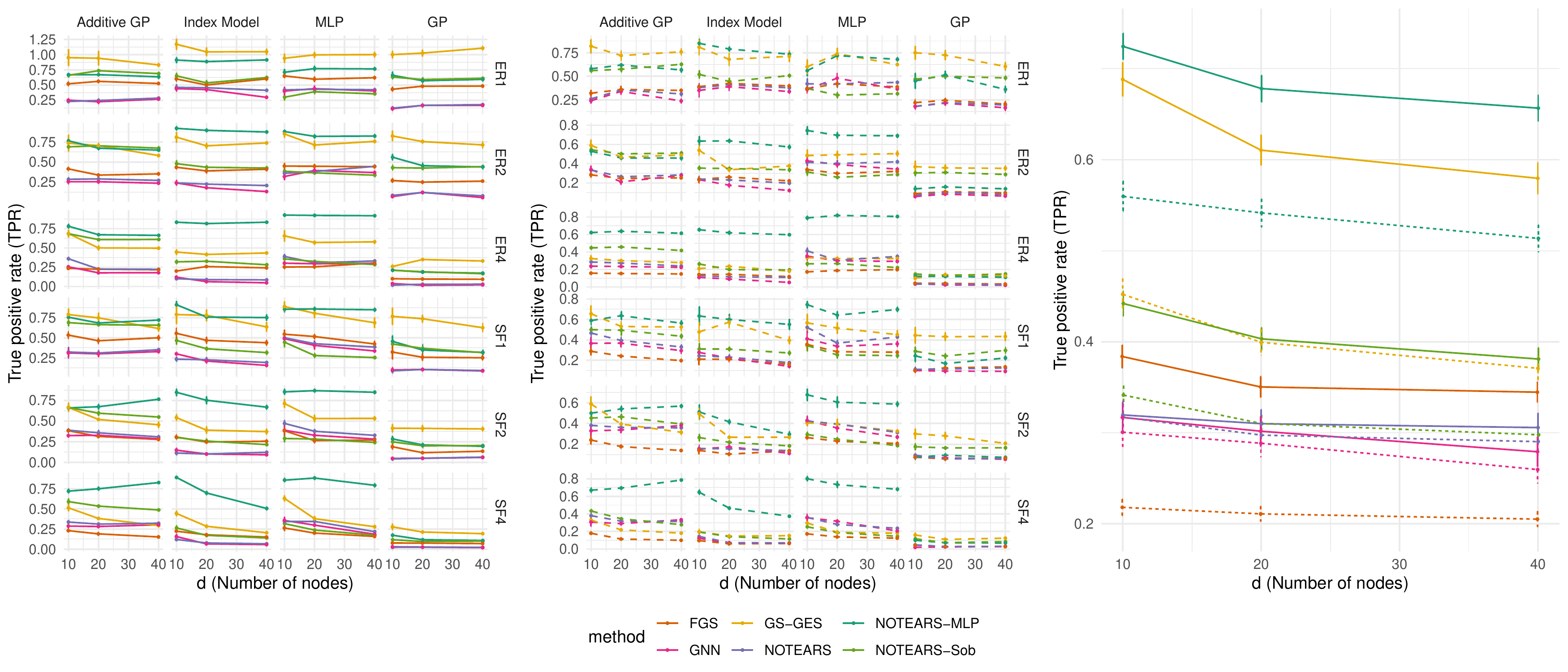}
\caption{Structure recovery measured by TPR (higher is better) to ground truth.}
\label{fig:compare:tpr}
\end{figure*}

\begin{figure*}[t]
\centering
\includegraphics[width=0.99\textwidth]{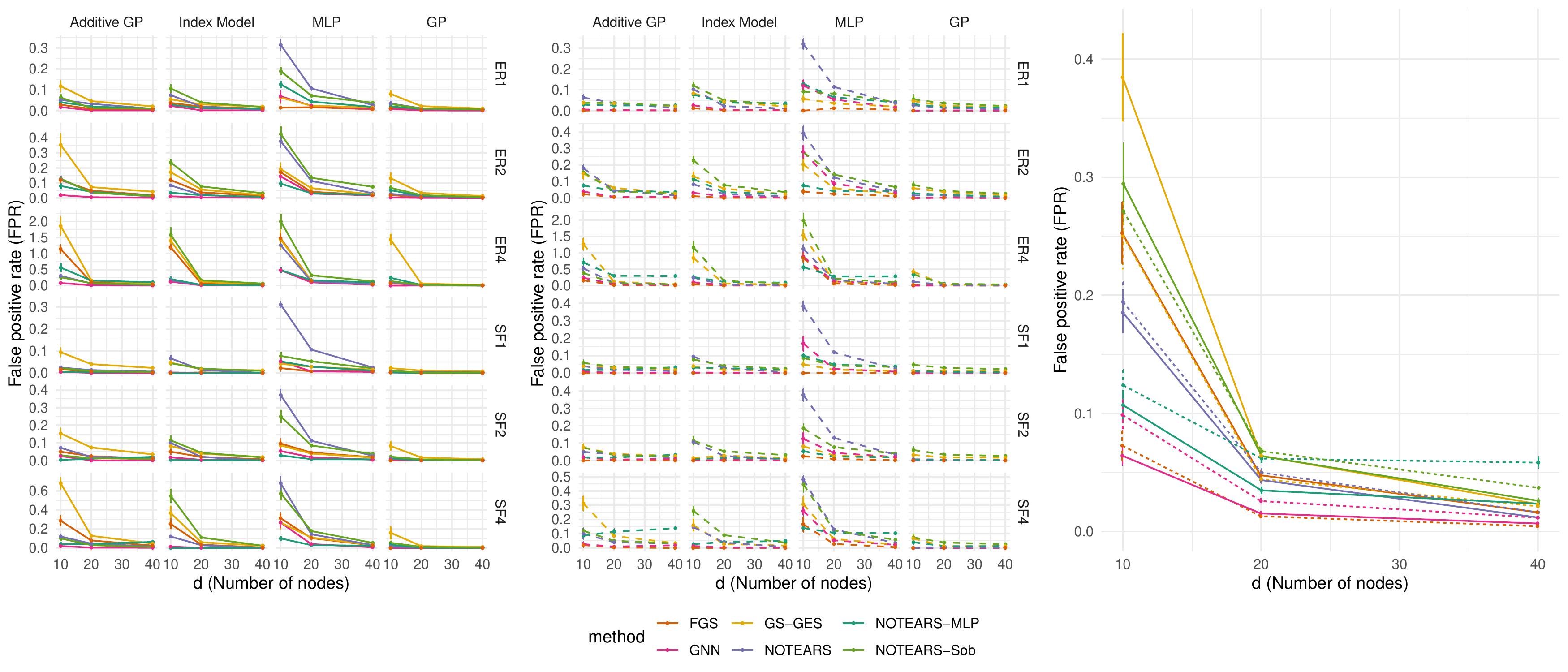}
\caption{Structure recovery measured by FPR (lower is better) to ground truth.}
\label{fig:compare:fpr}
\end{figure*}

\begin{figure*}[t]
\centering
\includegraphics[width=0.99\textwidth]{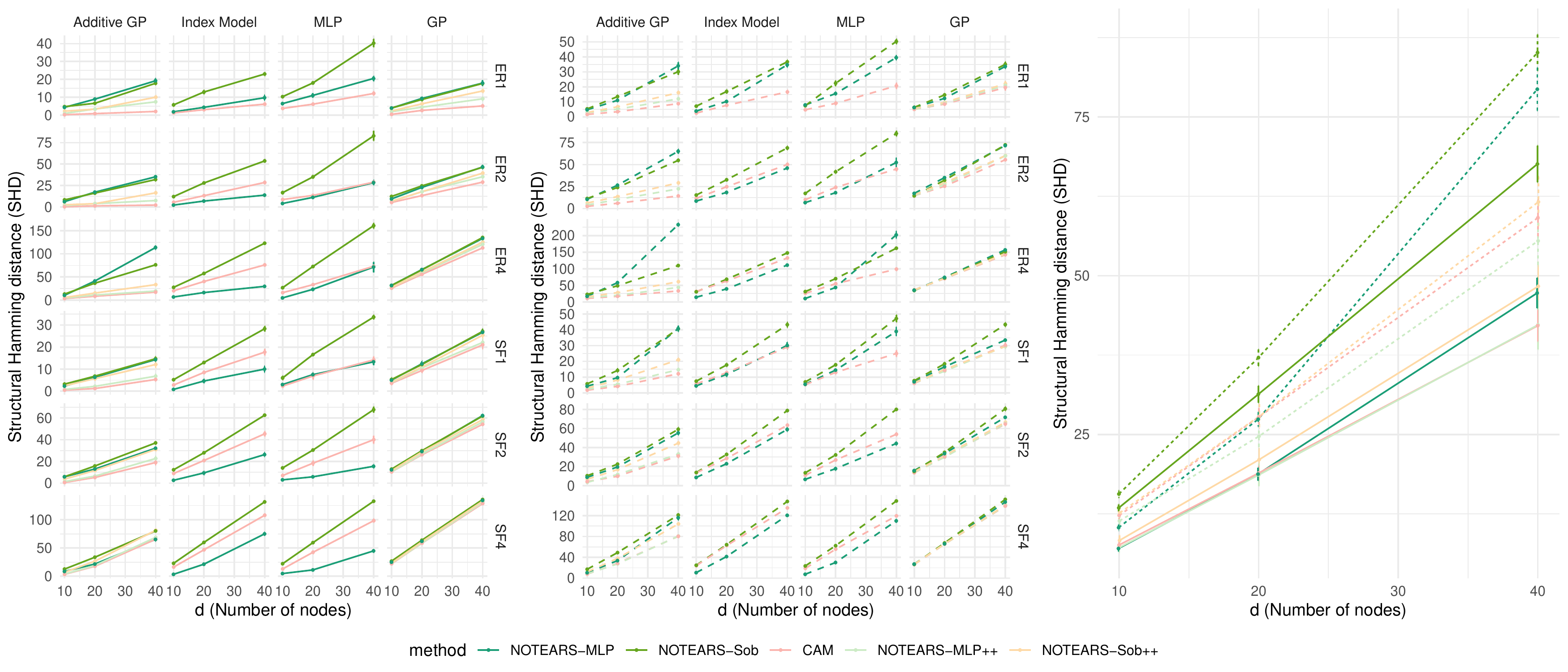}
\caption{Structure recovery measured by SHD (lower is better) to ground truth, compared with $ \cam $.}
\label{fig:compare:shd:cam:supp}
\end{figure*}

\begin{figure*}[t]
\centering
\includegraphics[width=0.99\textwidth]{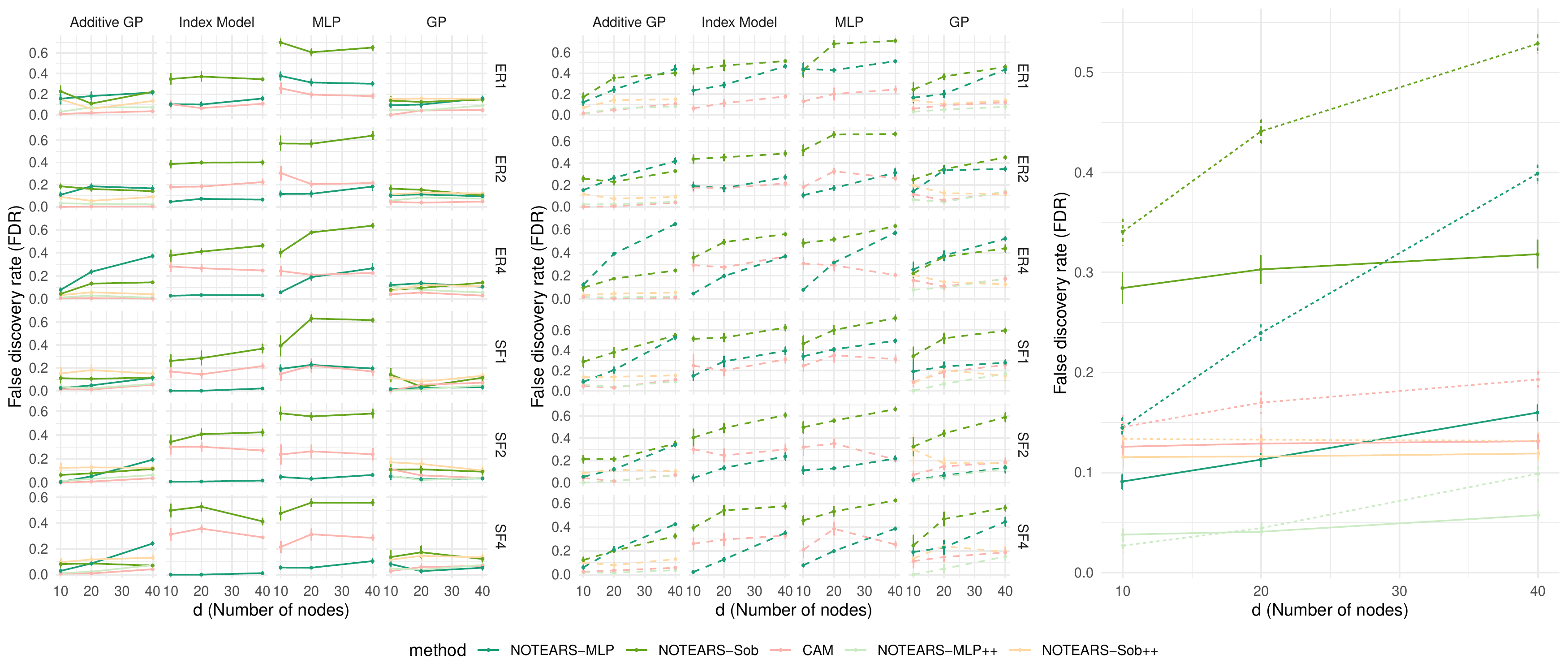}
\caption{Structure recovery measured by FDR (lower is better) to ground truth, compared with $ \cam $.}
\label{fig:compare:fdr:cam}
\end{figure*}

\begin{figure*}[t]
\centering
\includegraphics[width=0.99\textwidth]{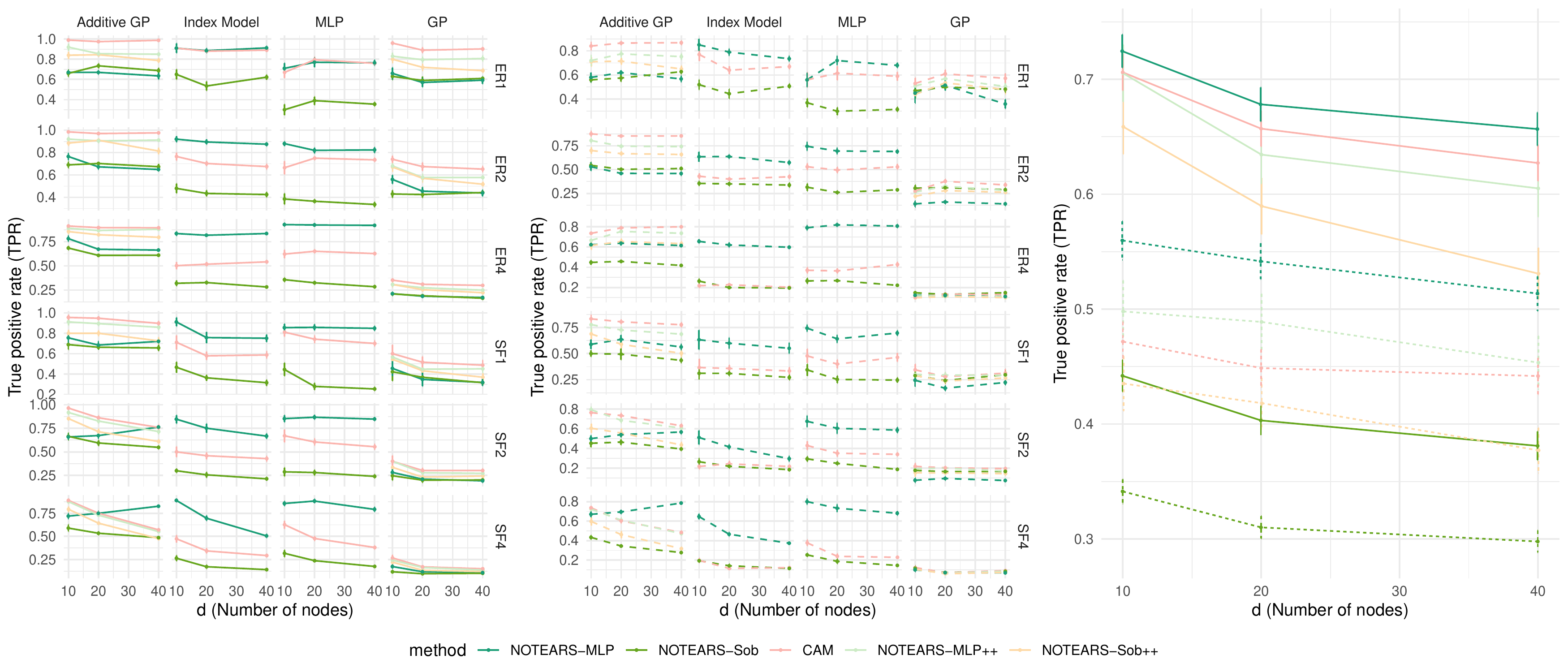}
\caption{Structure recovery measured by TPR (higher is better) to ground truth, compared with $ \cam $.}
\label{fig:compare:tpr:cam}
\end{figure*}

\begin{figure*}[t]
\centering
\includegraphics[width=0.99\textwidth]{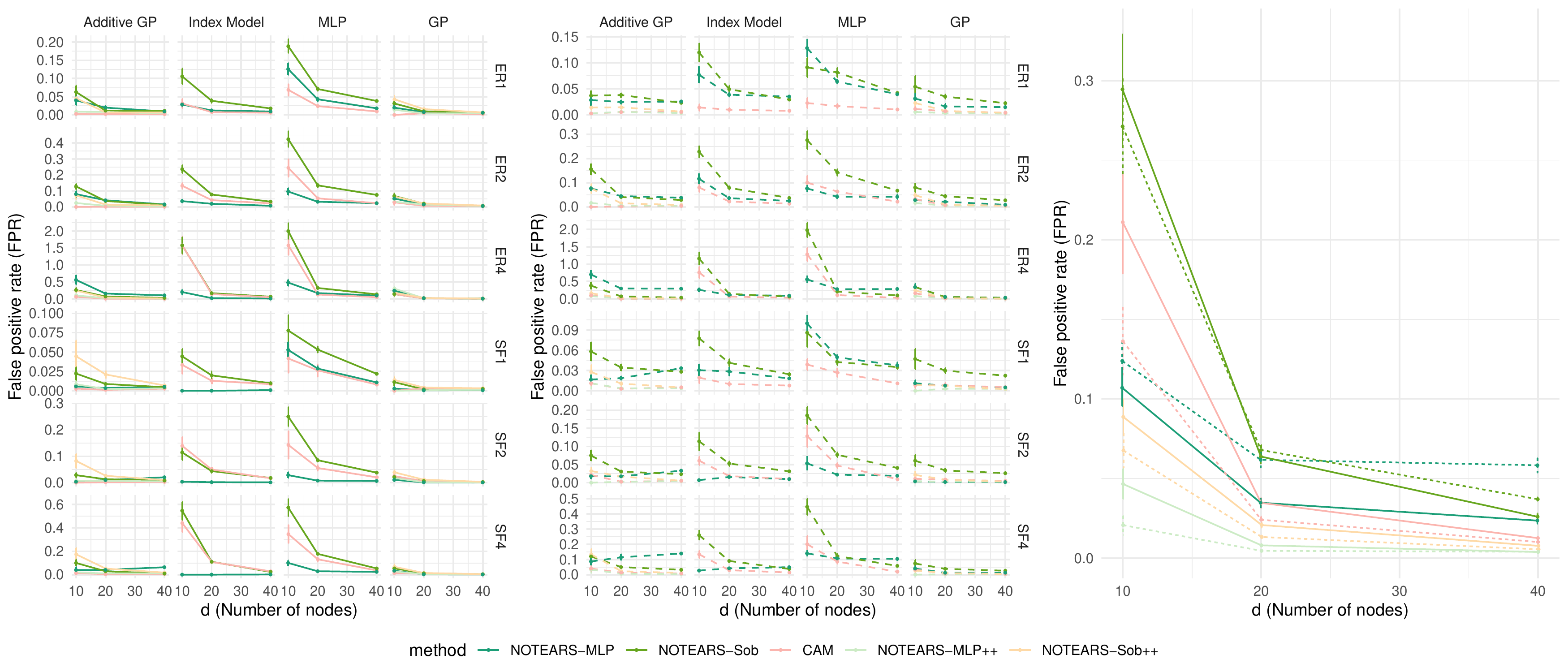}
\caption{Structure recovery measured by FPR (lower is better) to ground truth, compared with $ \cam $.}
\label{fig:compare:fpr:cam}
\end{figure*}

\end{document}